\documentclass[11pt]{article}

\usepackage{amssymb,amsmath,amsthm,sectsty,url,booktabs}
\usepackage[letterpaper,hmargin=1.0in,vmargin=1.0in]{geometry}
\usepackage[pdftex,colorlinks,linkcolor=blue,citecolor=blue,filecolor=blue,urlcolor=blue]{hyperref}
\usepackage{cleveref}
\usepackage{color}
\usepackage[boxed]{algorithm}
\usepackage{tikz}
\usepackage{graphicx}
\usepackage{nicefrac}
\usepackage{aliascnt}

\newtheorem{theorem}{Theorem}[section]
\crefname{theorem}{Theorem}{Theorems}

\newaliascnt{lemma}{theorem}
\newtheorem{lemma}[lemma]{Lemma}
\aliascntresetthe{lemma}
\crefname{lemma}{Lemma}{Lemmas}

\newaliascnt{proposition}{theorem}

\aliascntresetthe{proposition}
\crefname{proposition}{Proposition}{Propositions}

\newaliascnt{corollary}{theorem}

\aliascntresetthe{corollary}
\crefname{corollary}{Corollary}{Corollaries}

\newaliascnt{fact}{theorem}

\aliascntresetthe{fact}
\crefname{fact}{Fact}{Facts}

\newaliascnt{definition}{theorem}
\newtheorem{definition}[definition]{Definition}
\aliascntresetthe{definition}
\crefname{definition}{Definition}{Definitions}

\newaliascnt{remark}{theorem}
\newtheorem{remark}[remark]{Remark}
\aliascntresetthe{remark}
\crefname{remark}{Remark}{Remarks}

\newaliascnt{conjecture}{theorem}

\aliascntresetthe{conjecture}
\crefname{conjecture}{Conjecture}{Conjectures}

\newaliascnt{claim}{theorem}

\aliascntresetthe{claim}
\crefname{claim}{Claim}{Claims}

\newaliascnt{question}{theorem}

\aliascntresetthe{question}
\crefname{question}{Question}{Questions}

\newaliascnt{exercise}{theorem}

\aliascntresetthe{exercise}
\crefname{exercise}{Exercise}{Exercises}

\newaliascnt{example}{theorem}

\aliascntresetthe{example}
\crefname{example}{Example}{Examples}

\newaliascnt{notation}{theorem}

\aliascntresetthe{notation}
\crefname{notation}{Notation}{Notations}

\newaliascnt{problem}{theorem}

\aliascntresetthe{problem}
\crefname{problem}{Problem}{Problems}

\newcommand{\norm}[1]{\lVert#1\rVert}

\def\E{\mathbb E}

\newcommand{\N}{\mathbb N}
\newcommand{\R}{\mathbb R}

\newcommand{\Sph}{\mathbb{S}}
\newcommand{\unisph}{\mathcal{U}(d)}

\newcommand{\iu}{\mathbf{i}}

\newcommand{\kk}{\mathrm{k}}

\newcommand\numberthis{\addtocounter{equation}{1}\tag{\theequation}}

\title{Random Fourier Features in Near-Linear Time}

\author{
  Sepideh Mahabadi\\
  Microsoft Research -- Redmond\\
  \texttt{smahabadi@microsoft.com}
  \and
  Sandeep Silwal\\
  University of Wisconsin-Madison \\
  \texttt{silwal@cs.wisc.edu}\\
  \and 
  Tal Wagner\\
  Tel Aviv University \\
  \texttt{talwag@tauex.tau.ac.il}
}

\begin{document}

\maketitle

\begin{abstract}
Random Fourier features approximate a shift-invariant kernel in $\mathbb{R}^d$, such as the Gaussian kernel $\exp(-\tfrac12\|x-y\|_2^2)$, using independent samples from its spectral measure to construct an unbiased estimator for the kernel value with subgaussian concentration. 
However, generating $m$ such features requires $O(md)$ time. 
A long line of work on structured random features has accelerated this computation for the Gaussian kernel and other radial kernels, albeit with weaker concentration guarantees.
We close this gap by showing a random feature map that can be computed in near-linear time $O((d+m)\log d)$ and yields near-subgaussian concentration. 
We obtain it through a new and tighter analysis of the Fastfood feature map of Le, Sarl\'{o}s and Smola (ICML 2013), using Brascamp-Lieb inequalities and the Schoenberg representation of radial kernels. 
\end{abstract}

\section{Introduction}
A positive definite kernel $\kk(x,y)$, such as the Gaussian kernel $e^{-\|x-y\|_2^2/2}$, can be viewed as an inner product $\kk(x,y)=\phi(x)^T\phi(y)$ where $x\mapsto\phi(x)$ is a feature map into a space that may be infinite-dimensional. 
Kernel methods access this feature space implicitly through evaluations of $\kk$, without constructing the vectors $\phi(x)$ explicitly, an approach known as the ``kernel trick''. While widely useful, this flexibility can nevertheless be computationally costly: training over a dataset $\{x_i\}_{i=1}^n$ often requires storing or manipulating the $n\times n$ Gram matrix $(k(x_i,x_j))_{i,j}$, and evaluating a learned predictor on a test point $y$ may require computing $\kk(y,x_i)$ on many of the training points.

Random Fourier features (RFF) \cite{rahimi2007random}  address this problem for shift-invariant kernels by sampling $m$ independent frequencies from the kernel’s so called spectral measure (see \Cref{sec:rff}) and replacing the kernel by an inner product in a $O(m)$-dimensional randomized feature space. That is, we approximate $\kk(x,y)$ as the inner product $F_\kk(x)^TF_\kk(y)$ where $F_\kk$ is called the (randomized) feature map and maps to $O(m)$ dimensions. For every fixed pair $x,y$, this estimator is unbiased and admits strong (subgaussian) concentration as $m$ grows. 
RFF has become a cornerstone method in machine learning and continues to be widely used in modern AI systems and architectures (e.g., \cite{peng2021random,saito2025scalable,nguyen2026spectraformer,zhang2026flexformer}). 
However, its drawback is that computing $F_\kk(x)$ generally takes $O(md)$ time, which is a substantial computational bottleneck when the dimension $d$ is large.

As a result, a long line of work has been dedicated to faster variants of RFF. Most work has focused on the Gaussian and other radial kernels through the use of structured transforms, mainly randomized Hadamard transforms (RHTs). An RHT is a linear map $x\mapsto HDx$, where $H$ is the Hadamard matrix and $D$ is a random diagonal matrix. An RHT can be computed in time $O(d\log d)$, while at the same time it mimics certain properties of a truly random rotation, which would have taken time $O(d^2)$ to compute. This makes it useful in efficient methods for rotation-invariant kernels (i.e., radial). Many fast RFF variants have been designed by simulating pseudo-random Fourier frequencies via repeated RHTs \cite{le2013fastfood,yu2016orthogonal,choromanski2016recycling,choromanski2017unreasonable,cherapanamjeri2022uniform,backurs2024efficiently,wagner2026new}. 

This line of work has achieved a sequence of either provably or heuristically faster variants of RFF. However, it has not achieved an RFF variant with nearly linear time $\widetilde O(d+m)$ and nearly subgaussian concentration, leaving open the question whether such random features are possible.

\paragraph{Our results.} 
We answer this question affirmatively by proving the existence of an unbiased random feature map with time $O((d+m)\log d)$ and strong concentration, for the Gaussian kernel and other kernels. 
Our random feature  map is the simplified form of the Fastfood map due to \cite{le2013fastfood}, and our contribution is a new and tighter analysis of it. 
Indeed, it was noted in \cite{le2013fastfood} that the authors believed their analysis of Fastfood was pessimistic, and that a better tail bound could be provable through a different approach. 
Our work shows that that postulation was correct. 
The comparison between their bounds and ours, as well as other bounds, is shown in \Cref{tbl:main}. 

\paragraph{Kernel density estimation (KDE).} 
KDE is a canonical use case for random features. The problem is to preprocess a finite set $X\subset\R^d$ such that given any fixed query $y\in\R^d$, we could report an estimate of $\tfrac1{|X|}\sum_{x\in X}\kk(x,y)$ as quickly as possible. We consider \emph{additive} error up to $\pm\varepsilon$, with high probability. There is a very simple reduction from KDE to random features.
Given an unbiased and strongly concentrated feature map $F:\R^d\rightarrow\R^m$, we compute $F(X)=\tfrac1{|X|}\sum_{x\in X}F(x)$ during preprocessing,  and then report $F(X)^TF(y)$ given any $y$. The query time is $O(m)$ plus the time to compute $y\mapsto F(y)$ (i.e., the sum of the two rightmost columns in \Cref{tbl:main}). 
Our results yield the first KDE solution with query time $\widetilde O(d+1/\varepsilon^2)$ for this additive approximation guarantee, answering affirmatively the main open question posed in \cite{wagner2026new}. 
See \Cref{app:kde}. 

\paragraph{Bandwidth.} 
Kernels usually have a scaling parameter $\sigma$ called bandwidth (e.g., the Gaussian kernel is $\exp(-\norm{x-y}_2^2/(2\sigma^2))$. 
As a convention, we set $\sigma=1$ throughout. This does not limit generality, since all the kernels we discuss are functions of the Euclidean distance between the input pair of points. Therefore, our results extend to any bandwidth by appropriately scaling the input points.


\subsection{Formal statement of results}

Our main theorem is for the Gaussian kernel. 

\begin{table*}[t]
  \caption{Random feature maps $F:\R^d\rightarrow\R^m$ for the Gaussian kernel. The table lists the output dimension $m=m(\varepsilon,\delta)$ needed for the tail bound $\Pr[\left|\exp(-\frac12\norm{x-y}_2^2)-F(x)^TF(y)\right|\geq\varepsilon]\leq\delta$ for every fixed pair $x,y\in\R^d$, as well as the time it takes to compute $x\mapsto F(x)$. The dependence on $\delta$ is omitted since it is $\log^{O(1)}(1/\delta)$ in all the bounds. $\Delta$ denotes an upper bound on any $\norm{x-y}_2$.}
  \label{tbl:main}
  \begin{center}
        \begin{tabular}{llll}
          \toprule
          \textsc{Reference}  & \textsc{Method}         & \textsc{Out dim.}      & 
          \textsc{Time for $x\mapsto F(x)$} 
          \\
          \midrule
          \cite{rahimi2007random} & Random Fourier features & $O(1/\varepsilon^2)$ & $O(d/\varepsilon^2)$   \\
          \cite{le2013fastfood} & Fastfood &  $\widetilde O(\Delta^2/\varepsilon^2)$ & $\widetilde O(d+\Delta^2/\varepsilon^2)$   \\
           \cite{backurs2024efficiently} & Fast dim. reduction + RFF & $\widetilde O(1/\varepsilon^2)$ & $\widetilde O(d+1/\varepsilon^4)$    \\
            \cite{wagner2026new} & Fast spherical emb. + Fastfood & $\widetilde O(1/\varepsilon^3)$ & $\widetilde O(d+\varepsilon\Delta^2+1/\varepsilon^3)$         \\
          \midrule
          This work (Thm \ref{thm:main}) &Fastfood (new analysis)   &  $\widetilde O(1/\varepsilon^2)$ & $\widetilde O(d + 1/\varepsilon^2)$  \\
          \bottomrule
        \end{tabular}
  \end{center}
\end{table*}

\begin{theorem}\label{thm:main}
Let $d,m\geq1$ be integers. There is a randomized map $F:\R^d\rightarrow\R^{2m}$, samplable in time $O(d+m)$ and computable in time $O((d+m)\log d)$, such that for every $x,y\in\R^d$:
\begin{equation}\label{eq:mainthm_expectation}
  \E[F(x)^TF(y)] = \exp(-\norm{x-y}_2^2/2),
\end{equation}
and for all $t>0$ and $0<\eta<1$: 
\begin{equation}\label{eq:mainthm_concentration}
  \Pr\left[|F(x)^TF(y) - \exp(-\norm{x-y}_2^2/2)|\geq t\right]\leq \eta + 2\exp\left(-\frac{mt^2}{4\log(2d/\eta)}\right) .
\end{equation}
\end{theorem}

Our next theorem extends \Cref{thm:main} beyond the Gaussian kernel to other radial kernels under mild assumptions. A kernel $\kk:\R^d\times\R^d\rightarrow\R$ is called \emph{radial} if it depends only on the Euclidean distance between its inputs, that is, if there exists a map $\kappa:\R\rightarrow\R$ such that $\kk(x-y)=\kappa(\norm{x-y}_2)$ for every $x,y\in\R^d$. The kernel is called \emph{normalized} if $\kappa(0)=1$. We also need the kernel to be positive definite, which is defined formally in \Cref{sec:rff}.

\begin{theorem}\label{thm:radial}
Let $d,m\geq1$ be integers where $d$ is a power of 2. 
Let $\kk:\R^d\times\R^d\rightarrow[-1,1]$ be a continuous positive definite normalized radial kernel.
There is a randomized map $F_\kk:\R^d\rightarrow\R^{2m}$, computable in time $O((d+m)\log d)$, such that for every $x,y\in\R^d$:
\begin{equation}\label{eq:radial_expectation}
  \E[F_\kk(x)^TF_\kk(y)] = \kk(x,y),
\end{equation}
and for all $t>0$ and $0<\eta<1$: 
\begin{equation}\label{eq:radoial_concentration}
  \Pr\left[|F_\kk(x)^TF_\kk(y) - \kk(x,y)|\geq t\right]\leq \eta + 2\exp\left(-\frac{mt^2}{32\log(2d/\eta)}\right) + 2\exp\left(-\frac{mt^2}{8}\right) .
\end{equation}
The time to sample $F$ is $O(d+m(1+T^{\mathrm{rad}}_\kk)$, where $T^{\mathrm{rad}}_\kk$ is the time to draw a sample from the radial measure of $\kk$. 
\end{theorem}
The radial measure of the kernel is the radial part of its spectral measure. We review the spectral measure in \Cref{sec:rff} and define the radial measure formally in \Cref{sec:radial}. 
We show in \Cref{sec:radial_eg} that for many standard kernels, one can sample from the radial measure in expected time $O(1)$. 

Lastly, we show that the running time in \Cref{thm:main} is nearly optimal. 

\begin{theorem}\label{thm:lb}
Let $d\ge 1$, $0<\varepsilon\le 1/8$, and
$0<\delta\le \varepsilon^4$. Suppose a randomized map
$F:\mathbb{R}^d\to\mathbb{R}^m$ satisfies, for every fixed $x,y$,
\begin{equation}\label{eq:epsdelta}
\Pr\!\left[
  \left| F(x)^TF(y)
  -\exp(-\|x-y\|_2^2/2)\right|>\varepsilon
\right]\le\delta.
\end{equation}
Then, evaluating $F$ must take time $ \Omega(d + \frac{\log(1/\delta))}{\varepsilon^2\log(1/\varepsilon)})$.
\end{theorem}
If we set $t=\varepsilon$, $\eta=\frac12\delta$, $m=O(\frac{\log(1/\delta))\log(d/\delta)}{\varepsilon^2})$ in \Cref{thm:main}, \Cref{eq:mainthm_concentration} becomes \Cref{eq:epsdelta}. Thus, \Cref{thm:lb} shows that the computation time of $F$ in \Cref{thm:main} is optimal up to log factors.  

\subsection{Techniques}
We present a new analysis for the Fastfood transform \cite{le2013fastfood}. Fastfood is a sequence of two RHTs. The inner RHT acts as flattening on the input vector (i.e., spreading its $\ell_2$-mass approximately equally across its coordinates), as per a classical result of \cite{ailon2009fast} (see \Cref{lmm:fastjlflattening}). 
The challenge in analyzing the outer RHT is that, unlike original RFF, its output coordinates are not independent, thus requiring tail bound techniques for dependent random variables. Previous methods used concentration theorems for Lipschitz functions \cite{le2013fastfood} or Wiener chaoses \cite{wagner2026new} over the Gaussian diagonal entries of the outer RHT. Unfortunately, these techniques sustain a term that grows with the squared norm of the input vector, leading to the excess dependence on the diameter $\Delta$ as seen in  \Cref{tbl:main}. 

The main new step in our proof of \Cref{thm:main} is to control the outer RHT with the classic Brascanp-Lieb inequality \cite{brascamp1976best,lieb1990gaussian}. Roughly, it states that if our random variables arise as the projections of a Gaussian vector on $m$ ``sufficiently well-spread'' unit directions $\{w_i\}_{i=1}^m$, then although they are dependent, they concentrate like $c^*m$ truly independent random variables. $c^*\in(0,1)$ is a parameter that measures how well-spread the directions are (namely, $c^*\sum_iw_iw_i^T$ must be sub-isotropic; see \Cref{eq:subframe}). In our analysis, the output features are cast as the projections of the Gaussian diagonal of the outer RHT onto scaled Hadamard directions, and the flattening property of the inner RHT ensures that $c^*$ is sufficiently large, i.e., the directions are well-spread. 

\Cref{thm:radial} extends this approach beyond the Gaussian kernel by using Schoenberg's representation \cite{schoenberg1938metric,schoenberg1938metricb} of radial kernels as mixtures of spherical shell kernels (see \Cref{eq:schoenberg}). 
Then the above reasoning is applied to each shell kernel with a spherical variant of Brascamp-Lieb, where the projected vector is uniform over the unit sphere. 
The lower bound in \Cref{thm:lb} follows from standard arguments (e.g., \cite{chen2017relative}) as a consequence of Alon's rank bound for perturbed identity matrices \cite{alon2009perturbed} by designing an input dataset where the pairwise kernel values mimick a gram matrix of nearly orthogonal vectors.

\subsection{Related work}
Following \cite{rahimi2007random}, the methods summarized in \Cref{tbl:main} have been proposed as faster variants of random Fourier features. The Fastfood method, key to our work, was introduced in \cite{le2013fastfood} and analyzed through tail bounds for Lipschitz functions of Gaussians. \cite{backurs2024efficiently} proved new guarantees for the fast dimension reduction theorem of \cite{ailon2009fast} and used it as a pre-processing step for RFF. \cite{wagner2026new} proved a fast version of a spherical embedding theorem due to \cite{bartal2011dimensionality} and used it as pre-processing for Fastfood. The bounds arising from these methods are incomparable and depend on the parameter regime. All are based on RHTs, popularized in the algorithmic literature by \cite{sarlos2006improved,ailon2009fast}. 

The above methods yield a ``for-each'' guarantee that holds for each fixed pair with high probability. Fast RFF with a stronger ``for-all'' guarantee, which holds simultaneously for all pairs in a bounded domain, was studied in \cite{cherapanamjeri2022uniform}. 
In a different vein, \cite{yu2016orthogonal} proposed Structured Orthogonal Random Features (SORF), a fast variant of RFF that is also based on RHTs. Although highly effective empirically, it 
does not establish a kernel estimation tail bound comparable
to the methods in \Cref{tbl:main}. 

In addition to RFF, other approximate kernel feature methods include Quasi Monte-Carlo \cite{avron2016quasi,huang2024quasi}, quadrature \cite{bach2017equivalence,munkhoeva2018quadrature}, Nystr\"om \cite{rudi2015less,gittens2016revisiting}, and hashing-based methods \cite{andoni2009dimension,charikar2017hashing,siminelakis2019rehashing,backurs2019space,gorodissky2026locality}.

\section{Random Fourier features: technical background}\label{sec:rff}
We review the basics of RFF due to \cite{rahimi2007random}. 
Let $\kk:\R^d\times\R^d\rightarrow\R$ be a map that we refer to as a kernel over $\R^d$. 
The kernel is \emph{shift-invariant} if $\kk(x,y)$ depends only on $x-y$ for every $x,y\in\R^d$, that is, if $\kk(x,y)=\kk(x-y)$ with a slight abuse of notation.  
The kernel is \emph{positive definite} if $\sum_{i=1}^n\sum_{j=1}^nc_ic_j\kk(x_i,x_j)\geq0$ 
for every $n\in\N$, $c_1,\ldots,c_n\in\R$ 
and $x_1,\ldots,x_n\in\R^d$.  

By the classical Bochner theorem, a continuous shift-invariant kernel satisfying $\kk(0)=1$ is positive definite if and only if it is the Fourier transform of a probability measure $\mu_\kk$ over $\R^d$:
\begin{equation}\label{eq:bochner}
  \kk(x-y) = \E_{\omega\sim\mu_\kk}[e^{\iu\cdot\omega^T(x-y)}] ,
\end{equation}
where $\iu$ denotes the imaginary unit. $\mu_\kk$ 
is unique and called the \emph{spectral measure} of $\kk$. For example, for the Gaussian kernel $\kk(x,y)=\exp(-\norm{x-y}_2^2/2)$, the spectral measure is the standard multivariate Gaussian distribution, $N(0, I_d)$. 

The goal of random features is to find a map $F_\kk:\R^d\rightarrow\R^m$ such that $F_\kk(x)^TF_\kk(y)$ is a good approximation for $\kk(x,y)$ for each pair $x,y\in\R^d$. 
\cite{rahimi2007random} proposed to obtain this from \Cref{eq:bochner} via Monte-Carlo sampling. If one draws iid samples $\omega_1,\ldots,\omega_m$ from the spectral measure $\mu_\kk$, and defines $F_\kk:\R^d\rightarrow\R^{2m}$ by
\begin{equation}\label{eq:rff}
  F_\kk(x) = \frac1{\sqrt{m}}\oplus_{i=1}^m\left(\sin(\omega_i^Tx), \; \cos(\omega_i^Tx)\right) ,
\end{equation}
then from \Cref{eq:bochner} one can show that
\begin{equation}\label{eq:rffmean}
  \forall\;x,y\in\R^d\quad\kk(x,y) = \E_{\{\omega_i\}}[F_\kk(x)^TF_\kk(y)] .
\end{equation}
Furthermore, since the coordinates of $F$ are iid and bounded in $[-1,1]$, the standard Chernoff-Hoeffding tail bound yields strong concentration about the mean:
\begin{equation}\label{eq:rffconc}
  \forall\;x,y\in\R^d,\;\forall\;t>0,\quad \Pr[\left|F_\kk(x)^TF_\kk(y)-\kk(x,y)\right|\geq t]\leq2\exp(-mt^2/2) .
\end{equation}
Thus, once computed, random features can be used to speed kernel computations in a variety of downstream tasks. Unfortunately, computing $F_\kk(x)$ takes time $O(md)$, which could on its own be computationally heavy, particularly where the dimension $d$ is large. The goal of the present paper, and of the long line of work it belongs to, is to achieve properties similar to \Cref{eq:rffmean,eq:rffconc} for natural kernels with a faster computation time. 

\begin{remark}
    Another possible computational bottleneck is to sample $F_\kk$ to begin with (before applying it to any point). This takes time $O(m(1+T^{\mathrm{spec}}_\kk))$, where $T^{\mathrm{spec}}_\kk$ is the time to draw a single sample from $\mu_\kk$. 
    For many natural kernels, including the Gaussian kernel, $T^{\mathrm{spec}}_\kk=\Theta(d)$. 
\end{remark}

\begin{remark}
    Another qualitatively similar way to define RFF is to also draw uniform iid $\beta_1,\ldots,\beta_m\sim[0,2\pi]$ and define $\breve F_\kk(x) = \sqrt{2m^{-1}}\oplus_{i=1}^m\cos(\omega_i^Tx+\beta_i)$. Both forms of RFF appear widely in the literature. We focus on \Cref{eq:rff} for simplicity, since our analysis comes out cleaner without the additional source of randomness $\{\beta_i\}$.
\end{remark}

\section{Fast random features for the Gaussian kernel (\Cref{thm:main})}

We assume that $d$ is a power of 2 w.l.o.g.~by padding with $0$s.
We also focus on the case $m\leq d$, since that is the challenging regime. Once this regime is proven, the  extension to $m>d$ then follows as detailed in \Cref{app:bigm}.  

Our map is the simplified form of the Fastfood map from \cite{le2013fastfood}.  
We denote:
\[
  V := \sqrt{d}HGHB ,
\]
where $H\in\R^{d\times d}$ is the normalized Hadamard matrix (its entries are $\pm1/\sqrt{d}$), $B\in\R^{d\times d}$ is a diagonal matrix with iid random signs on the diagonal, and $G\in\R^{d\times d}$ is a random diagonal matrix with iid $N(0,1)$s on the diagonal. The map $F:\R^d\rightarrow\R^{2m}$ is defined as
\[
  F(x) = \frac{1}{\sqrt m}\oplus_{i=1}^m\left(\sin\left((Vx)_i\right),\cos\left((Vx)_i\right)\right) .
\]

\subsection{Preliminaries}
The next lemmas are either standard or shown in prior work.  
\begin{lemma}[Fast JL flattening \cite{ailon2009fast}]
    \label{lmm:fastjlflattening}
For every $z\in\R^d$ and $\eta>0$, 
    \[ \Pr_B\left[\norm{HBz}_\infty \leq \norm{z}_2\sqrt{\frac{2\log(2d/\eta)}{d}}\right] > 1-\eta . \]
\end{lemma}

\begin{lemma}\label{lmm:fastfood_uznorm}
For every $z\in\R^d$ and every supported $B$,  $\norm{HBz}_2=\norm{z}_2$. 
\end{lemma}
\begin{proof}
    Since $H$ is orthogonal and $B$ only flips some signs of $z$,  $\norm{HBz}_2=\norm{Bz}_2=\norm{z}_2$.
\end{proof}

\begin{lemma}[\cite{le2013fastfood}]\label{lmm:fastfood_gaussian}
For every $z\in\R^d$, conditioned on $B$, each entry of $Vz$ is marginally distributed like $N(0,\norm{z}^2_2)$. 
\end{lemma}
\begin{proof}
    Conditioned on $B$, the vector $u=HBz$ is fixed and all randomness in $V$ is encompassed in $G$. 
    Since $G$ has iid gaussians on its diagonal, and $\sqrt{d}H$ has all entries in $\pm1$, each row of $\sqrt{d}HG$ is a vector of iid gaussians (while the different rows are not independent). Therefore, by the spherical invariance of the normal distribution, each entry of $Vz=\sqrt{d}HGu$ is distributed like $N(0,\norm{u}_2^2)$, which per \Cref{lmm:fastfood_uznorm} is $N(0,\norm{z}_2^2)$.
\end{proof}

\begin{lemma}[standard fact]\label{lmm:cosnormal}
    $\E_{X\sim N(0,\nu^2)}[\cos(X)]=e^{-\nu^2/2}$. 
\end{lemma}

\begin{lemma}[Hoeffding's lemma]\label{lmm:hoeffding}
Let $X$ be a real-valued random variable such that $X\in[a,b]$ w.p.~1. Then for all $\lambda>0$, 
$\E[e^{\lambda (X-\E X)}] \leq e^{\lambda^2(b-a)^2/8}$.
\end{lemma}

\begin{lemma}[rank-1 Gaussian Brascamp-Lieb; see \Cref{app:bl}]\label{lmm:bl_gaussian}
Let $w_1,\ldots,w_k\in\R^d$ be unit vectors, let $c_1,\ldots,c_k>0$, and suppose
\begin{equation}\label{eq:subframe}
    \sum_{i=1}^k c_iw_iw_i^T \preceq I_d .
\end{equation}
Then, for all non-negative measurable $f_i:\R\rightarrow[0,\infty)$,
\begin{equation}\label{eq:bl_conclusion}
\E_{g\sim N(0,I_d)}\left[\prod_{i=1}^kf_i(w_i^Tg)^{c_i}\right] \leq \prod_{i=1}^k\left(\E[f_i(N(0,1))]\right)^{c_i} .
\end{equation}
\end{lemma}

\subsection{Proof of \Cref{thm:main}}\label{sec:mainproof}
Fix a pair $x,y\in\R^d$ and let $z:=x-y$.  
We start by proving \Cref{eq:mainthm_expectation}.

\begin{lemma}\label{lmm:ipcos}
$F(x)^TF(y) = \frac1m\sum_{i=1}^m\cos((Vz)_i)$.
\end{lemma}
\begin{proof}
\begin{align*}
    F(x)^TF(y) &= \frac1m\sum_{i=1}^m\left(\sin((Vx)_i)\sin((Vy)_i) + \cos((Vx)_i)\cos((Vy)_i)\right) \\
    &= \frac1m\sum_{i=1}^m\left(\cos((Vx)_i - (Vy)_i)\right)  \\
    &= \frac1m\sum_{i=1}^m\cos((Vz)_i) ,
\end{align*}
by a standard trigonometric sum-product formula for the middle equality. 
\end{proof}

By \Cref{lmm:fastfood_gaussian}  $(Vz)_i\sim N(0,\norm{z}_2^2)$ for every $i$, thus \Cref{eq:mainthm_expectation} follows from \Cref{lmm:cosnormal}.

We proceed to prove \Cref{eq:mainthm_concentration}.  We first outline the argument and the use of Brascamp-Lieb (\Cref{lmm:bl_gaussian}). 
Suppose $X_1,\ldots,X_m$ are zero-centered random variables and $X=\frac1m\sum_iX_i$. Our goal is to prove a tail bound of the form $\Pr[X\geq t]\leq f(m,t)$. 
The standard Chernoff-Hoeffding proof blueprint goes as follows: for every $s>0$, we have $\Pr\left[X\geq t\right] 
  = \Pr\left[e^{sX}\geq e^{st}\right] 
  \leq e^{-st}\E\left[e^{sX}\right]$ by Markov's inequality. Thus the tail bound reduces to bounding the moment generating function (MGF) $\E[e^{sX}]$. In the simplest case, if $\{X_i\}$ are independent, we have
\begin{equation}\label{eq:chernoffiid}
  e^{-st}\E\left[e^{\frac{s}{m}\sum_iX_i}\right]
  = e^{-st}\E\left[\prod_ie^{\frac{s}{m}X_i}\right]
  = e^{-st}\prod_i\E\left[e^{\frac{s}{m}X_i}\right].
\end{equation}
Then, a tail bound can be obtained by bounding the MGFs of each individual $X_i$ (often through Hoeffding's lemma) and optimizing for the choice of $s$ that produces the tightest bound. 

In our case, $\{X_i\}$ would be associated with the output features of the structured random feature map $F$ in \Cref{thm:main}. Unfortunately, they are not independent, since every output feature depends on all the random diagonal entries of $B$ and $G$. Fortunately, the Brascamp-Lieb inequality \Cref{lmm:bl_gaussian} provides an alternative way to bound the MGF, which can replace the rightmost inequality in \Cref{eq:chernoffiid}. 
It requires that the $X_i$s are determined 
by the projection of a single standard Gaussian vector $g$ onto directions $w_1,\ldots,w_m$,
and that we can assign them weights $c_1,\ldots,c_m$ to render them sub-isotropic, meaning to satisfy \Cref{eq:subframe}. Intuitively, this condition ensures that no direction is overrepresented, and $\{c_i\}$ quantify how overrepresented any direction in $\R^d$ could be by the collection $\{w_i\}$. 
With a common weight $c_i=c^*$, Brascamp-Lieb allows us to bound the MGF of the $m$ terms as if they were approximately $c^*m$ independent random variables. 

This matches our setting because conditioned on $B$, every output feature of $F$ is 
determined by $w_i^Tg$ where $g$ is the diagonal vector of $G$ and the $w_i$s are fixed directions. 
Namely, for a fixed $u\in\R^d$, the outer Hadamard stage, $HG$, satisfies $HGu=H\mathrm{diag}(u)g$. Since the rows of $H$ are orthonormal, the sub-isotropy condition boils down to bounding the magnitude of entries in $u$. By Fast JL flattening (\Cref{lmm:fastjlflattening}), the inner Hadamard stage, $HB$, ensures that $u=HBz$ is an orthogonal image of $z$ that with high probability has entry magnitude nearly as small as possible. Taken together, the two Hadamard stages in $F$ allow us to utilize Brascamp-Lieb instead of the fully independent factorization in the rightmost equality of \Cref{eq:chernoffiid}.

We now turn to the formal proof. We may assume that $z\neq0$, since if $z=0$ then $F(x)^TF(y)=1$ by \Cref{lmm:ipcos}, rendering \Cref{eq:mainthm_concentration} trivial. We let
\[
  X_i = \cos((Vz)_i)-\exp(-\norm{z}_2^2/2) ,
\]
and observe that by \Cref{lmm:ipcos},
$F(x)^TF(y) - \exp(-\norm{z}_2^2/2) = \frac1m\sum_{i=1}^mX_i$.
Conditioned on $B$, as per the foregoing discussion, for every $s>0$ we have by Markov's inequality,
\[
\Pr\left[\frac{1}{m}\sum_iX_i\geq t \;|\; B\right] 
  = \Pr\left[e^{\frac{s}{m}\sum_iX_i}\geq e^{st} \;|\; B\right] 
  \leq
  e^{-st}\E\left[\prod_ie^{\frac{s}{m}X_i} \;|\; B\right] ,
\]
which, by writing out $\{X_i\}$, becomes
\begin{align*}
    &\Pr\left[F(x)^TF(y) - \exp(-\norm{z}_2^2/2)\geq t \;|\; B\right] \\
    &\leq
    \exp(-st) \cdot \E\left[\prod_{i=1}^m \exp\left(\frac{s}{m}\left(\cos((Vz)_i)- \exp(-\norm{z}_2^2/2))\right)\right)\;|\; B\right]. \numberthis\label{eq:hoeffding_mid}
\end{align*}

We now use Brascamp-Lieb (\Cref{lmm:bl_gaussian}) to bound the right-hand side. 
Recall that $u=HBz$. 
Let $g=(g_1,\ldots,g_d)\sim N(0,I_d)$ be the diagonal entries of $G$.   
Observe that the $i$th entry of $Vz$ can be written as
\[
  (Vz)_i = e_i^TVz = e_i^T(\sqrt{d}HGu) = e_i^T(\sqrt{d}H\mathrm{diag}(g)u) = e_i^T(\sqrt{d}H\mathrm{diag}(u)g) .
\]
Thus, if we let
\[
  w_i = \frac{\sqrt{d}}{\norm{z}_2}\mathrm{diag}(u)H^Te_i ,
\]
then
\begin{equation}\label{eq:w}
  (Vz)_i = \norm{z}_2\cdot w_i^Tg . 
\end{equation}
First, we verify that $w_i$ has unit length. Since $\sqrt{d}H^Te_i$ is a vector with all entries in $\{\pm1\}$, 
\[
  \norm{w_i}_2 = \frac{1}{\norm{z}_2}\norm{\mathrm{diag}(u)\cdot\sqrt{d}H^Te_i}_2 = \frac{\norm{u}_2}{\norm{z}_2} = 1,
\]
having used \Cref{lmm:fastfood_uznorm}. 
Next, we choose $(c_i)$ to satisfy the sub-isotropy condition \Cref{eq:subframe}. We have,
\[
  \sum_{i=1}^mw_iw_i^T \leq \sum_{i=1}^dw_iw_i^T = \frac{d}{\norm{z}_2^2}\sum_{i=1}^d\mathrm{diag}(u)H^Te_ie_i^TH\mathrm{diag}(u)
\]
\[
  = \frac{d}{\norm{z}_2^2}\mathrm{diag}(u)H^T\left(\sum_{i=1}^de_ie_i^T\right)H\mathrm{diag}(u) \\
  = \frac{d}{\norm{z}_2^2}\cdot\mathrm{diag}(u^2) \\
  \preceq \frac{d\norm{u}_\infty^2}{\norm{z}_2^2} \cdot I_d .
\]
Therefore, letting each $c_i$ be equal to
\begin{equation}\label{eq:cstar}
c^* = \frac{\norm{z}_2^2}{d\norm{u}_\infty^2} , 
\end{equation}
we have \Cref{eq:subframe}. 
Finally, we define 
\[
  f_i(\alpha) = \exp\left(\frac{s}{c^*\cdot m}\left( \cos(\norm{z}_2\cdot\alpha) -\exp\left(-\norm{z}_2^2/2\right)\right)\right) .
\]
Then by \Cref{eq:w},
\[
  f_i(w_i^Tg)^{c^*} = \exp\left(\frac{s}{m}\left( \cos((Vz)_i) -\exp\left(-\norm{z}_2^2/2\right)\right)\right) .
\]
This completes the setup for Brascamp-Lieb. \Cref{lmm:bl_gaussian} yields
\begin{equation}\label{eq:bl_gaussian_apply}
     \E\left[\prod_{i=1}^m \exp\left(\frac{s}{m}\left(\cos((Vz)_i)- \exp(-\norm{z}_2^2/2))\right)\right) \;|\; B \right]  
    \leq  \prod_{i=1}^m\left(\E[f_i(N(0,1))]\right)^{c^*}.
\end{equation}
To control each individual $\E[f_i(N(0,1))]$, 
let $Z$ be the random variable $Z=\cos(\norm{z}_2\cdot N(0,1))$. It is clearly bounded in $[-1,1]$ w.p.~1, and by \Cref{lmm:cosnormal}, $\E[Z]= \exp\left(-\norm{z}_2^2/2\right)$. Therefore, by Hoeffding's lemma (\Cref{lmm:hoeffding}),
\begin{align*}
   \E[f_i(N(0,1))] &=  \E\left[\exp\left(\frac{s}{c^*\cdot m}\left( \cos(\norm{z}_2\cdot N(0,1)) -\exp\left(-\norm{z}_2^2/2\right)\right)\right)\right] \\
   &= \E\left[\exp\left(\frac{s}{c^*\cdot m}\left(Z - \E[Z]\right)\right)\right] \leq \exp\left(\frac{s^2}{2(c^*)^2m^2}\right) .
\end{align*}
Plugging this in \Cref{eq:bl_gaussian_apply}
yields
\begin{equation}\label{eq:mgf_gaussian}
    \E\left[\prod_{i=1}^m \exp\left(\frac{s}{m}\left(\cos((Vz)_i)- \exp(-\norm{z}_2^2/2))\right)\right) \;|\; B\right]
    \leq  \exp\left(\frac{s^2}{2c^*m}\right) .
\end{equation}
Plugging into \Cref{eq:hoeffding_mid} and choosing $s=tc^*m$, we get
\[
  \Pr\left[F(x)^TF(y) - \exp(-\norm{z}_2^2/2)\geq t \;|\; B\right] 
  \leq  \exp\left(-t^2c^*m/2\right) .
\]
The proof of the lower tail is similar, thus
\begin{equation}\label{eq:gaussian_tail_condb}
  \Pr\left[|F(x)^TF(y) - \exp(-\norm{z}_2^2/2)|\geq t\;|\; B\right] 
  \leq  2\exp\left(-t^2c^*m/2\right) .
\end{equation}
Finally, we bound $c^*$ by un-conditioning on $B$. By the Fast JL flattening lemma (\Cref{lmm:fastjlflattening}), we have $\norm{u}_\infty^2\leq\norm{z}_2^2\cdot\frac{2\log(2d/\eta)}{d}$ w.p.~$1-\eta$ over the random choice of $B$. Thus, by \Cref{eq:cstar}, $c^*\geq1/(2\log(2d/\eta))$. Plugging this in \Cref{eq:gaussian_tail_condb} with a union bound over the failure of the flattening event (losing an additive $\eta$ to the overall failure probability on the right-hand side), we get
\[
  \Pr\left[|F(x)^TF(y) - \exp(-\norm{z}_2^2/2)|\geq t\right] 
  \leq  \eta + 2\exp\left(-t^2m/(4\log(2d/\eta))\right) ,
\]
which is \Cref{eq:mainthm_concentration}. This concludes the proof of \Cref{thm:main}.

\subsection{The $m>d$ regime}\label{app:bigm}
\Cref{sec:mainproof} proves \Cref{thm:main} in the case $m<d$. In this section we show that the extension for $m>d$ follows readily: Suppose $m>d$, and write $m=\sum_{\ell=1}^qm_j$ with $q=\lceil m/d\rceil$ and $1\leq m_j\leq d$ for all $j$. 
Sample the map $F^*:\R^d\rightarrow\R^{2m}$ as follows:
\begin{itemize}
    \item Sample a random diagonal matrix $B\in\R^{d\times d}$ with iid random signs on the diagonal.
    \item For every $j=1,\ldots,q$, sample a random diagonal matrix $G_j\in\R^{d\times d}$ with iid standard Gaussians on the diagonal. The $\{G_j\}$ are independent. 
    \item Denote $V_j = \sqrt{d}HG_jH_B$. 
    Define $F_j:\R^d\rightarrow\R^{2m_j}$ as the map from \Cref{sec:mainproof}:
    \[
      F_j(x) = \frac{1}{\sqrt{m_j}}\oplus_{i=1}^{m_j}\left(\sin((V_jx)_i),\;\cos((V_jx)_i)\right) .
    \]
\end{itemize}
$F^*$ is now defined as $F^* = \oplus_{i=j}^q\sqrt{\frac{m_j}m}F_j .$ Fix a pair $x,y\in\R^d$ and let $z=x-y$. 
For brevity, denote
\[
  \Delta_j = \frac{1}{m_j}\sum_{i=1}^{m_j}\left(\cos((V_jz)_i)- \exp(-\norm{z}_2^2/2))\right).
\]
Then by the MGF bound in \Cref{eq:mgf_gaussian},
\[
  \E[\exp(s\Delta_j) \;|\; B]
    \leq  \exp\left(\frac{s^2}{2c^*m_j}\right)
\]
Let  
\[
\Delta^*=\sum_{j=1}^q\frac{m_j}{m}\Delta_j . \] 
Similarly to \Cref{lmm:ipcos}, one can show that 
\[
  F^*(x)^TF^*(y) = \frac{1}{m}\sum_{j=1}^q\sum_{i=1}^{m_j}\cos((V_jz)_i) ,
\]
and therefore
\begin{equation}\label{eq:deltastar}
    \Delta^* = F^*(x)^TF^*(y) - \exp(-\norm{z}_2^2/2) .
\end{equation}
Observe that $\{V_j\}$ are independent conditioned on $B$, and also that $c^*$ is mutual to all $\{F_j\}$, since it depends only on $B$ and $z$. Therefore,  
\begin{align*}
  \E[\exp(s\Delta^*) \;|\; B]
    &= \E\left[\exp\left(s\sum_{j=1}^q\frac{m_j}{m}\Delta_j\right) \;|\; B\right] \\
    &= \prod_{j=1}^q\E\left[\exp\left(\frac{sm_j}{m}\Delta_j\right) \;|\; B\right] \\
    &\leq \prod_{j=1}^q\exp\left(\frac{(sm_j/m)^2}{2c^*m_j}\right) \\
    &= \exp\left(\frac{s^2}{2c^*m^2}\sum_{j=1}^qm_j\right) \\
    &= \exp\left(\frac{s^2}{2c^*m}\right) .
\end{align*}
Plugging this MGF bound in the Hoeffding-Chernoff blueprint and using \Cref{eq:deltastar}, we get that for every $s,t>0$, 
\[
  \Pr\left[  F^*(x)^TF^*(y) - \exp(-\norm{z}_2^2/2) > t \;|\; B\right] \leq \exp\left(-st + \frac{s^2}{2c^*m}\right) .
\]
From this point, the proof is finished in the same way as \Cref{sec:mainproof}: we choose $s=mc^*t$ and de-condition on $B$ with the JL flattening event from \Cref{lmm:fastjlflattening}. This proves the upper tail bound. The proof of the lower tail is similar and we omit it.

\section{Fast random features for other radial kernels (\Cref{thm:radial})}\label{sec:radial}

We first define the radial measure of a kernel. 

\begin{definition}[radial measure of a kernel]
    Let $\kk:\R^d\times\R^d\rightarrow[-1,1]$ be a shift-invariant positive definite kernel, and let $\mu_\kk$ be its spectral measure over $\R^d$.  The \emph{radial measure} $\nu_\kk$ of $\kk$ is defined as the probability distribution of $\norm{\omega}_2$ where $\omega\sim\mu_\kk$.   
\end{definition}

Thus, the radial measure can be realized by drawing a sample from the spectral measure and taking its Euclidean norm; however, this may not be the most efficient way to sample from the radial measure, as we discuss in \Cref{sec:radial_eg}. If the kernel is radial, then the angular part of its spectral measure is uniform on the unit sphere, and thus the radial measure fully ``encodes'' the kernel.

The random map $F_\kk$ in \Cref{thm:radial} is a close variant of the map from \Cref{sec:mainproof}. 
Let $R_1,\ldots,R_m\sim\nu_k$ be iid samples from the radial measure of $\kk$. 
Using the notation from \Cref{sec:mainproof}, 
we denote 
$  \bar{V} = \norm{G}_F^{-1} V = \norm{G}_F^{-1}\sqrt{d}HGHB$.
The map $F_\kk$ is defined as
\[
  F_\kk(x) = \frac{1}{\sqrt m}\oplus_{i=1}^m\left(\sin\left(R_i(\bar{V}x)_i\right),\cos\left(R_i(\bar{V}x)_i\right)\right) .
\]
Note that $B$ and $G$ take time $O(d)$ to sample, and $R_1,\ldots,R_m$ take time $O(m(1+T_{\nu}))$ to sample, and therefore the total time to sample $F_\kk$ is $O(d+m(1+T_{\nu}))$.

\subsection{Preliminaries}
Throughout this section, we let $\unisph$ denote the uniform distribution over the unit sphere $\Sph^{d-1}$.
The following is the classical Schoenberg representation for positive definite radial kernels \cite{schoenberg1938metric,schoenberg1938metricb}, which expresses them as mixtures of spherical shell kernels of different radii.

\begin{lemma}[Schoenberg representation]\label{lmm:schoenberg}
    For every fixed dimension $d$, there exists a map $\Omega_d:\R_{\geq0}\rightarrow[-1,1]$ (called the spherical shell kernel) that satisfies
    \[
         \forall\;x,y\in\R^d, \quad \Omega_d(\norm{x-y}_2)=\E_{\Theta\sim\unisph}[\cos\left(\Theta^T(x-y)\right)] .
    \]
    If $\kk(x,y)=\kappa(\norm{x-y}_2)$ is a continuous positive definite normalized radial kernel, then
    \begin{equation}\label{eq:schoenberg}
         \forall\;x,y\in\R^d, \quad  \quad\kk(x,y) = \E_{R\sim\nu_k}[\Omega_d(R\cdot\norm{x-y}_2)] ,
    \end{equation}
    where $\nu_\kk$ is the radial measure of $\kk$. Consequently, 
    \[
      \forall\;x,y\in\R^d, \quad \kk(x,y) = \E_{R\sim\nu_k,\;\Theta\sim\unisph}[\cos(R\cdot\Theta^T(x-y))] .
    \]
\end{lemma}

\Cref{lmm:schoenberg} is classical and appears in various phrasings in the literature. 
We briefly point out how to obtain it from Theorem 1.2 in \cite{golinskii2018radial}. Their equation (1.2) in our notation is $\kappa(r) = \E_R[\Omega_d(rR)]$ for all $r\geq0$,
and by their equation (1.4),
$\Omega_d(\norm{z}_2) = \E_\Theta[e^{\iu\cdot\Theta^Tz}]$ for all $z\in\R^d$. 
Since $e^{\iu\zeta}=\cos(\zeta)+\iu\sin(\zeta)$ and $\sin$ is symmetric about zero, the latter is equivalent to $\Omega_d(\norm{z}_2) = \E_\Theta[\cos(\Theta^Tz)]$. 
Hence also $\Omega_d(R\cdot\norm{z}_2) = \E_\Theta[\cos(R\cdot\Theta^Tz)]$ for all $R\geq0$. Plugging $z=x-y$ and $r=\norm{z}_2$ yields the lemma. 

Next is a spherical variant of the Brascamp-Lieb inequality from \Cref{lmm:bl_gaussian}.

\begin{lemma}[rank-1 spherical Brascamp-Lieb; see \Cref{app:bl}]\label{lmm:sphbl}
Under the same conditions as \Cref{lmm:bl_gaussian}, for all non-negative measurable $f_i:\R\rightarrow[0,\infty)$,
\[
\E_{\Theta\sim\unisph}\left[\prod_{i=1}^kf_i(w_i^T\Theta)^{c_i/2}\right] \leq \prod_{i=1}^k\left(\E_{\Theta\sim\unisph}[f_i(w_i^T\Theta)]\right)^{c_i/2} .
\]
\end{lemma}

\subsection{Proof of \Cref{thm:radial}}
Like in the proof of \Cref{thm:main}, we assume that $d\leq m$. The extension to $m>d$ then follows by the same arguments as those in \Cref{app:bigm}.

Fix a pair $x,y\in\R^d$ and $z=x-y$. 
Let $g=(g_1,\ldots,g_d)\sim N(0,I_d)$ be the diagonal of $G$. 
Let $ \bar{g}:=\norm{g}_2^{-1}\cdot g .$ By the spherical invariance of the normal distribution, $\bar{g}$ is distributed like $\unisph$. We start by proving \Cref{eq:radial_expectation}. The following lemma follows from the same proof as in \Cref{lmm:ipcos}.
\begin{lemma}\label{lmm:ipcos_radial}
$F_\kk(x)^TF_\kk(y) = \frac1m\sum_{i=1}^m\cos(R_i( \bar{V}z)_i)$.
\end{lemma}

The following lemma is analogous to \Cref{lmm:fastfood_gaussian}.
\begin{lemma}\label{lmm:sphere_marginal}
Each row of $\bar{V}$ is marginally distributed like $\unisph$. 
\end{lemma}
\begin{proof} 
Fixing a row index $i$, let $h_i\in\{\pm1\}^d$ be the $i$th column of $\sqrt{d}H$. Then the $i$th row of $\bar{V}$ can be written as
$e_i^T\bar{V} 
  = h_i^T\norm{G}_F^{-1}GHB = \bar{g}^T\mathrm{diag}(h_i)HB$.
Since $H$ is an orthogonal matrix then so is $\mathrm{diag}(h_i)HB$.
Therefore, since $\bar{g}$ is distributed like $\unisph$, so is $\bar{g}^T\mathrm{diag}(h_i)HB$.  
\end{proof}

By \Cref{lmm:ipcos_radial}, $F_\kk(x)^TF_\kk(y)$ is the average of random variables of the form $\cos(R\cdot\Theta^Tz)$, where $R$ is sampled from $\nu_\kk$ and $\Theta$ from $\unisph$. By \Cref{lmm:schoenberg}, the mean of each such random variable is $\kk(x,y)$, and \Cref{eq:radial_expectation} follows. 

We proceed to prove the concentration bound in \Cref{eq:radoial_concentration}, whose proof at this point is rather similar to that of the Gaussian kernel case in \Cref{sec:mainproof}. 
Assume $z\neq 0$ since otherwise \Cref{eq:radoial_concentration} holds trivially.
Let $s,t>0$. Conditioning on $B$ and $\{R_i\}$ (so that the only remaining source of randomness in $F_\kk$ is $G$), we have by Markov's inequality followed by \Cref{lmm:ipcos_radial} (analogously to \Cref{eq:hoeffding_mid}),
\begin{align*}
    &\Pr\left[F_\kk(x)^TF_\kk(y) - \frac1m\sum_{i=1}^m\Omega_d(R_i\norm{z}_2)\geq t \;|\; B,\{R_i\}\right] \\
    &= \Pr\left[\exp\left(s\left(F_\kk(x)^TF_\kk(y) - \frac1m\sum_{i=1}^m\Omega_d(R_i\norm{z}_2)\right)\right)\geq\exp(st)\;|\; B,\{R_i\}\right] \\
    &\leq \exp(-st)\cdot\E\left[\exp\left(s\left(F_\kk(x)^TF_\kk(y) - \frac1m\sum_{i=1}^m\Omega_d(R_i\norm{z}_2)\right)\right)\;|\; B,\{R_i\}\right] \\
    &= \exp(-st)\cdot\E\left[\prod_{i=1}^m\exp\left(\frac{s}{m}\left(\cos\left(R_i(\bar{V}z)_i\right) - \Omega_d\left(R_i\norm{z}_2\right)\right)\right) \;|\; B,\{R_i\}\right]. \numberthis\label{eq:hoeffding_mid_radial}
\end{align*}

We now apply the spherical Brascamp-Lieb inequality (\Cref{lmm:sphbl}). 
We use the same $w_i$, $c_i$ and $c^*$ defined in \Cref{sec:mainproof}. There we showed they satisfy the conditions of \Cref{lmm:bl_gaussian}, which are the same as the conditions of \Cref{lmm:sphbl}. We define the maps $f_i$ as
\[
  f_i(\alpha) := \exp\left(\frac{2}{c^*}\cdot\frac{s}{m}\left(\cos\left(R_i\norm{z}_2\cdot\alpha\right) - \Omega_d\left(R_i\norm{z}_2\right)\right)\right) .
\]
Recalling from \Cref{eq:w} that $(Vz)_i=\norm{z}_2w_i^Tg$, and that $\bar{V}=\norm{G}_F^{-1}V=\norm{g}_2^{-1}V$ and $\bar{g}=\norm{g}_2^{-1}g$, we have $w_i^T\bar{g}=\norm{z}_2^{-1}(\bar{V}z)_i$.
Consequently (recalling that $c_i=c^*$ for all $i$),
\[
  f_i(w_i^T\bar{g})^{c_i/2} = \exp\left(\frac{s}{m}\left(\cos\left(R_i(\bar{V}z)_i\right) - \Omega_d\left(R_i\norm{z}_2\right)\right)\right) .
\]
Since $\bar{g}\sim\unisph$, by applying \Cref{lmm:sphbl} we get that the right-hand of \Cref{eq:hoeffding_mid_radial} is at most
\begin{equation}\label{eq:bl_radial_apply}
  \exp(-st)\cdot\prod_{i=1}^m\left(\E_{\Theta\sim\unisph}[f_i(w_i^T\Theta)]\right)^{c^*/2} .
\end{equation}

Next we control each $\E_{\Theta\sim\unisph}[f_i(w_i^T\Theta)]$. 
Let $X_i$ be the random variable $X_i=\cos(R_i(\bar{V}z)_i)$. By \Cref{lmm:sphere_marginal}, the $i$th row of $\bar{V}$ is distributed like $\Theta$, hence $X_i$ is distributed like $\cos(R_i\Theta^Tz)$. By \Cref{lmm:schoenberg}, $\E[X_i]=\Omega_d(R_i\norm{z}_2)$. 
Furthermore, $X_i$ is clearly bounded in $[-1,1]$. 
Therefore, by Hoeffding's lemma (\Cref{lmm:hoeffding}),
\[
    \E_\Theta[f_i(w_i^T\Theta)] = \E_\Theta\left[\exp\left(\frac{2s}{c^*m}\left(X_i - \E[X_i]\right)\right)\right]
    \leq \exp\left(\frac{2s^2}{(c^*)^2m^2}\right) .
\]
Plugging back this into \Cref{eq:bl_radial_apply,eq:hoeffding_mid_radial}, and choosing $s:=tc^*m/2$,
\[
    \Pr\left[F_\kk(x)^TF_\kk(y) - \frac1m\sum_{i=1}^m\Omega_d(R_i\norm{z}_2)\geq t \;|\; B,\{R_i\}\right] \leq
  \exp\left(-st + \frac{s^2}{c^*m}\right) \leq
  \exp\left(-\frac{c^*mt^2}{4}\right) . 
\]
The proof of the lower tail is similar, thus
\[
    \Pr\left[\left|F_\kk(x)^TF_kk(y) - \frac1m\sum_{i=1}^m\Omega_d(R_i\norm{z}_2)\right|\geq t \;|\; B,\{R_i\}\right] \leq
  2\exp\left(-\frac{c^*mt^2}{4}\right) . 
\]

Similarly to \Cref{sec:mainproof}, by Fast JL flattening (\Cref{lmm:fastjlflattening}) we have $c^*\geq1/(2\log(2d/\eta))$ with probability $1-\eta$ over the choice of $B$. Therefore, 
\[
    \Pr\left[\left|F_\kk(x)^TF_\kk(y) - \frac1m\sum_{i=1}^m\Omega_d(R_i\norm{z}_2)\right|\geq t \;|\; \{R_i\}\right] \leq
  \eta + 2\exp\left(-\frac{mt^2}{8\log(2d/\eta)}\right) . 
\]
Finally, we handle the term $\frac1m\sum_{i=1}^m\Omega_d(R_i\norm{z}_2)$. Consider each $Y_i:=\Omega_d(R_i\norm{z}_2)$ as a push-forward random variable of $R_i$. By \Cref{lmm:schoenberg}, $Y_i$ is bounded in $[-1,1]$ and $\E_{R}[Y_i]=\kk(x,y)$. Therefore, by the standard Chernoff-Hoeffding inequality, 
\[ \Pr\left[\left|\frac1m\sum_{i=1}^m\Omega_d(R_i\norm{z}_2) - \kk(x,y)\right|\geq t\right] = \Pr\left[\left|\frac1m\sum_{i=1}^mY_i - \kk(x,y)\right|\geq t\right]\leq 2\exp\left(-\frac{mt^2}{2}\right) . 
\]
By the triangle inequality and a union bound,
\[
    \Pr\left[\left|F_\kk(x)^TF_\kk(y) - \kk(x,y)\right|\geq 2t \right] \leq
  \eta + 2\exp\left(-\frac{mt^2}{8\log(2d/\eta)}\right) + 2\exp\left(-\frac{mt^2}{2}\right). 
\]
Scaling $t$ to $t/2$ yields \Cref{eq:radoial_concentration}, completing the proof of \Cref{thm:radial}.

\subsection{Specific radial kernels beyond gaussian}\label{sec:radial_eg}

We give explicit samplers for the radial measures of several standard
kernels, yielding efficient implementations and corollaries of \Cref{thm:radial}. The standard kernels we consider are the exponential,
rational quadratic, and Mat\'ern kernels. For each of these kernels,
sampling from the radial measure requires sampling only two independent scalar Gamma random variables. We assume sampling such Gamma random variables take $O(1)$ expected time as in \cite{marsaglia2000simple}.  Consequently, \Cref{thm:radial} gives
randomized feature maps for these kernels with expected sampling time
$O(d+m)$ and evaluation time $O((d+m)\log d)$, while retaining the
unbiasedness and concentration guarantees of the theorem.

We write $\kk(x,y)=\kappa(r)$, where
$r=\norm{x-y}_2$. In the following examples, we can set the bandwidth equal to $1$ without loss
of generality.
For $a>0$, let $\operatorname{Gamma}(a,1)$ denote the Gamma
distribution with shape parameter $a$ and rate parameter $1$.

Recall from Definition~4.1 that the radial measure $\nu_\kk$
is the distribution of $\norm{\omega}_2$ when
$\omega\sim\mu_\kk$, where $\mu_\kk$ is the spectral measure of
$\kk$. We describe how to generate a scalar radius $R\sim\nu_\kk$,
which supplies each of the independent radii $R_1,\ldots,R_m$
used in the construction of the feature map $F_\kk$ of
\Cref{thm:radial}.

Our samplers follow from the following observation. If a nonnegative
random variable $S$ satisfies
\[
  \kappa(r)=\E_S\left[\exp(-Sr^2/2)\right]
  \qquad\text{for all }r\geq0,
\]
then drawing $Q\sim\operatorname{Gamma}(d/2,1)$ independently of $S$
and setting $R=\sqrt{2SQ}$ gives a sample from $\nu_\kk$.
Indeed, for $Z\sim N(0,I_d)$ independent of $S$, the vector
$\sqrt{S}Z$ has distribution $\mu_\kk$, and
$\norm{Z}_2^2/2$ has distribution $\operatorname{Gamma}(d/2,1)$.
In the constructions below, we generate $S$ from a
scalar Gamma variable $A$, whose parameters depend on the kernel
and which is sampled independently of $Q$.

\paragraph{Exponential kernel.}
The exponential kernel is defined as $\kk(x,y)=\exp(-\norm{x-y}_2).$ For $A\sim\operatorname{Gamma}(1/2,1)$, integrating its density gives
\[
  \E_A\left[\exp\left(-\frac{r^2}{4A}\right)\right]
  =\frac{1}{\sqrt{\pi}}\int_0^\infty
    a^{-1/2}\exp\left(-a-\frac{r^2}{4a}\right)\,da
  =\exp(-r),
\]
where the final equality follows from standard calculations.
Thus we take $S=1/(2A)$, and a sample from the radial measure is
$R=\sqrt{Q/A}$.

\paragraph{Rational quadratic kernel.}
For a shape parameter $\alpha>0$, the rational quadratic kernel is
\[
  \kk(x,y)=
  \left(1+\frac{\norm{x-y}_2^2}{2\alpha}\right)^{-\alpha}.
\]
Taking $A\sim\operatorname{Gamma}(\alpha,1)$, its Laplace transform
gives
$\E_A[\exp(-Ar^2/(2\alpha))]
=(1+r^2/(2\alpha))^{-\alpha}$.
Hence $S=A/\alpha$, and a sample from the radial measure is
$R=\sqrt{2QA/\alpha}$.
The case $\alpha=1$ is the Cauchy kernel
$\kk(x,y)=(1+\norm{x-y}_2^2/2)^{-1}$ in this parameterization.

\paragraph{Mat\'ern kernel.}
For a smoothness parameter $\beta>0$, the Mat\'ern kernel is
defined by
\[
  \kappa(r)=
  \frac{2^{1-\beta}}{\Gamma(\beta)}
  \left(\sqrt{2\beta}\,r\right)^\beta
  K_\beta\left(\sqrt{2\beta}\,r\right),
  \qquad r>0,
\]
with $\kappa(0)=1$ by continuity. Here $\Gamma$ is the Gamma
function and $K_\beta$ is the modified Bessel function of the
second kind.
The integral representation of $K_\beta$
(see \cite{DLMF}) gives
$\kappa(r)=\E_A[\exp(-\beta r^2/(2A))]$ for
$A\sim\operatorname{Gamma}(\beta,1)$.
Thus $S=\beta/A$, and a sample from the radial measure is
$R=\sqrt{2\beta Q/A}$.
The exponential kernel is the special case $\beta=1/2$.

Each sampler uses two independent scalar Gamma samples.
Using the standard Gamma sampler of \cite{marsaglia2000simple},
each radius takes expected $O(1)$ time under the usual model of
unit-cost real arithmetic, elementary functions, and scalar
uniform and Gaussian sampling.
For every $R_i$, we draw a fresh independent pair $(Q,A)$,
independently of $B$ and $G$, the random sign and Gaussian diagonal
matrices defined in \Cref{sec:mainproof}.

Consequently, the radial sampling cost
$T^{\mathrm{rad}}_\kk$ in \Cref{thm:radial} is expected $O(1)$
for each of these three kernels. The theorem therefore gives a
randomized map $F_\kk:\R^d\rightarrow\R^{2m}$ with expected
sampling time $O(d+m)$ and evaluation time $O((d+m)\log d)$.
For every fixed pair $x,y\in\R^d$, the estimator
$F_\kk(x)^TF_\kk(y)$ is unbiased for $\kk(x,y)$ and satisfies
the concentration bound of \Cref{thm:radial}.

\section{Lower bound: Proof of \Cref{thm:lb}}\label{app:lb}
We will show a lower bound on the output dimension $m$ in \Cref{thm:lb}.
\begin{lemma}\label{lmm:lb}
    Under the assumptions of \Cref{thm:lb}, $m=\Omega(\frac{\log(1/\delta))}{\varepsilon^2\log(1/\varepsilon)})$. 
\end{lemma}

\Cref{thm:lb} follows from \Cref{lmm:lb} because computing $F$ requires at least reading its $d$ input coordinates (this argument is very simple and omitted here) and writing its output $m$ coordinates, thus overall $\Omega(d+m)$ time.

The proof is \Cref{lmm:lb} is by the following classical theorem on the rank of perturbed identity matrices.

\begin{theorem}[\cite{alon2009perturbed}]\label{thm:alonrank}
   Let $M\in\R^{n\times n}$ such that $|M_{ii}|\geq\frac12$ for all $i$ and $|M_{ij}|\leq\alpha$ for all $i\neq j$. 
   Let $\alpha\in[\frac{1}{2\sqrt n},\frac14)$. 
   Then $\mathrm{rank}(M)\geq\Omega(\frac{\log n}{\alpha^2\log(1/\alpha)})$.
\end{theorem}

\begin{proof}[Proof of \Cref{lmm:lb}]
Let $n=\lceil 1 / 2\sqrt{\delta}\rceil$. Let
$x_i=i\sqrt{2\log(2/\varepsilon)}\,e_1$ for $i\in[n]$, where $e_1$ is the first standard basis vector in $\R^d$.
The pairwise Gaussian kernel values given by $K_{ij} = \exp(- \|x_i - x_j\|_2^2/2)$ satisfy 
\begin{equation}{\label{eq:lb}}
 \forall i, \; K_{ii}=1 \quad \text{and} \quad \forall i\neq j, \; 0\le K_{ij}\le\varepsilon/2 .
\end{equation}
Since $n^2\delta\le 1/2$, a union bound gives the existence of a function $F$ whose Gram matrix 
$G_{ij}= F(x_i)^TF(x_j)$ satisfies, by combining \Cref{eq:lb,eq:epsdelta}
\[
   \forall i, \; G_{ii}\ge 1-\varepsilon\geq\frac12 \quad \text{and} \quad \forall i\neq j, \; |G_{ij}|\le 3\varepsilon/2 .
\]
Furthermore, $\mathrm{rank}(G)\leq m$. Choosing $\alpha=3\varepsilon/2$, the assumptions of \Cref{thm:lb} guarantee that
\[
  \frac{1}{2\sqrt n} \leq \delta^{1/4} \leq \varepsilon \leq \alpha \leq \frac{1}{4} .
\]
Therefore, by \Cref{thm:alonrank}, 
\[
  m \geq \mathrm{rank}(G) \geq \Omega\left(\frac{\log(1/\delta)}{\varepsilon^2\log(1/\varepsilon)}\right),
\]
as desired.
\end{proof}

\section{Kernel density estimation}\label{app:kde}
Let $\kk:\R^d\times\R^d\rightarrow
\R$ be a kernel. 
Given a subset $X\subset\R^d$ of size $n>0$, the kernel density estimation (KDE) function is defined as $KDE_X(y) = \frac1n\sum_{x\in X}\kk(x,y)$. 
A KDE data structure has the following operations:
\begin{itemize}
    \item Initialize: given $X\subset\R^d$, initialize the data structure.
    \item Query: given $y\in\R^d$, return an approximation $\widetilde{K}_X(y)$ of $KDE_X(y)$.
\end{itemize}
The operations may randomized. 
We say that the data structure satisfies $(\varepsilon,\delta)$-additive approximation if after initialization with any $X$ it satisfies
\[
  \forall \; y\in\R^d,\quad \Pr\left[ \left| KDE_X(y) - \widetilde{K}_X(y)\right| > \varepsilon \right] \leq \delta ,
\]
where the probability is over the internal randomness of the KDE data structure. 

The following result shows the existence of a Gaussian KDE data structure that satisfies $(\varepsilon,\delta)$-additive approximation with query time $\tilde O(d+\varepsilon^{-2})$. It is a corollary of \Cref{thm:main}. It answers the main open question posed in \cite{wagner2026new}; see that reference there for an account of previous bounds for this problem. We note that there is also a distinct line of work for KDE data structures with multiplicative error. The algorithms and bounds with this type of error are different than the pure additive error setting of this section, and we refer to \cite{charikar2017hashing, backurs2019space, charikar2020kernel} for details about this different line of approach.

\begin{theorem}\label{thm:kde}
    For all $\varepsilon,\delta\in(0,1/2)$, there is a KDE data structure for the Gaussian kernel with $(\varepsilon,\delta)$-additive error, with the following computational bounds:
    \begin{itemize}
        \item initialization time: $O(|X|(d+m)\log d)$,
        \item memory size: $O(d+m)$,
        \item query time: $O((d+m)\log d)$,
    \end{itemize}
    where $m=O(\varepsilon^{-2}\log(d/\varepsilon\delta)\log(d/\varepsilon\delta))$. 
\end{theorem}

\begin{proof}
Assume w.l.o.g.~that $d$ is a power of $2$. 
Let $m\geq16\varepsilon^{-2}\log(16d/\varepsilon\delta)\log(16/\varepsilon\delta))$. 
To initialize the data structure, we draw a randomized map $F:\R^d\rightarrow\R^{2m}$ from \Cref{thm:main} and compute $F(X)=\frac{1}{|X|}\sum_{x\in X}F(x)$. This takes time $O(|X|(d+m)\log d)$ by the sampling and evaluation times of $F$ in \Cref{thm:main}. We store the map $F$ (this takes $O(d+m)$ memory) and the value $F(x)$ (this takes $O(m)$ memory). At query time, given $y\in\R^d$, we compute $F(y)$ (time $O((d+m)\log d)$) and return $\widetilde{K}_X(y)=F(X)^TF(y)$ (time $O(m)$). 

To show the $(\varepsilon,\delta)$-additive approximation guarantee, we choose $p=\varepsilon\delta/4$ and $\eta=p/2$. Then by \Cref{eq:mainthm_concentration},
\begin{equation}\label{eq:kdeconc}
  \forall \; x\in X, \quad \Pr\left[\left| F(x)^TF(y) - \exp(-\norm{x-y}_2^2/2)\right| > \varepsilon/2\right] \leq p .  
\end{equation}
Denote,
\[
  E(x) = \left| F(x)^TF(y) - \exp(-\norm{x-y}_2^2/2)\right| \quad \text{and} \quad Z = \frac{1}{|X|}\sum_{x\in X}\mathbf1\{E(x)>\varepsilon/2\} .
\]
Then $\E[Z]\leq p$ by \Cref{eq:kdeconc}, and hence by Markov's inequality, $\Pr[Z>\varepsilon/4]\leq4p/\varepsilon=\delta$. Next, observe that $\norm{F(x)}_2^2=1$ for all $x\in\R^d$ (since $F(x)$ has the form $\frac1{\sqrt m}\oplus_{j=1}^m(\sin(\theta_j),\cos(\theta_j))$ for some $\{\theta_j\}$), and hence $|F(x)^TF(y)|\leq1$ for all $x,y$ by Cauchy-Schwartz. Furthermore, $\exp(-\norm{x-y}_2^2/2)\in[0,1]$ for all $x,y$. Consequently, $|E(x)|\leq2$ for every $x\in X$. Therefore, for every fixed $y$,
\begin{align*}
    \left| \widetilde{K}_X(y) - KDE_X(y) \right| &=  \left| \frac1{|X|}\sum_{x\in X}\left(F(x)^TF(y) - \exp(-\norm{x-y}_2^2/2)\right) \right|  \\
    &\leq \frac1{|X|}\sum_{x\in X}E(x)  \\
    &= \frac1{|X|}\sum_{x\in X:\;E(x)\leq\varepsilon/2}E(x) + \frac1{|X|}\sum_{x\in X:\;E(x)>\varepsilon/2}E(x)  \\
    &\leq \frac{\varepsilon}{2} + 2Z .
\end{align*}
Since $\Pr[Z>\varepsilon/4]\leq\delta$, the theorem follows. 
\end{proof}

A similar theorem follows for KDE with other radial kernels by \Cref{thm:radial}, with the corresponding radial-sampling cost during initialization. The query time and memory usage remain unchanged. 

\section{Differential privacy with KDE function release}
Differentially private KDE has been studied in a long line of work including \cite{hall2013differential,wang2016differentially,alda2017bernstein,coleman2020one,wagner2023fast,backurs2024efficiently,wagner2025learning,wagner2026new,gorodissky2026locality}. We start with the definitions. 

\begin{definition}[Differential Privacy \cite{dwork2006calibrating,dwork2014algorithmic}]
    Let $\epsilon>0$. Let $M$ be a randomized algorithm that maps datasets with elements in $\mathcal{X}$ to a space of outputs $\mathcal{O}$. $M$ is called an $\epsilon$-differentially private ($\epsilon$-DP) mechanism if for every two datasets $X,X'$ such that $X'$ is obtained from $X$ by either adding or dropping a single element, and for every set of outputs $O\subset\mathcal O$, it holds that $\Pr[M(X)\in O]\leq e^\epsilon\Pr[M(X')\in O]$.
\end{definition}

\begin{definition}[DP KDE in the function release model]
    Let $\epsilon>0$ and $\alpha,\eta\in(0,1)$. Let $\kk:\mathcal X\times\mathcal X\rightarrow[0,1]$ be a kernel. In the DP KDE problem in the function release model, the goal is to design an $\epsilon$-DP mechanism $M$ that given a dataset $X\subset\mathcal X$, returns a function $\widetilde{F}_X:\mathcal{X}\rightarrow\R$ that for each $y\in\mathcal X$ satisfies $\Pr[|\widetilde{F}_X(y)-KDE_X(y)|\leq\alpha]\geq1-\eta$. 
\end{definition}

The fastest DP KDE mechanisms for the Gaussian kernel in the high-dimensional regime are based on random Fourier features or faster variants of them \cite{wagner2023fast,backurs2024efficiently,wagner2026new}. \Cref{thm:main} can be used to subsume them and obtain the following result.

\begin{theorem}[Differentially private Gaussian KDE]
\label{thm:dp-kde}
Let $d\ge 1$, $\epsilon>0$, and $\alpha,\eta\in(0,1/2)$. Let
\[
 L=\log\frac{6}{\eta},
 \qquad
 m=\left\lceil \frac{C_0}{\alpha^2}
       \log^2\!\left(\frac{2d}{\alpha\eta}\right)\right\rceil,
\]
where $C_0$ is a sufficiently large absolute constant.
There is an $\epsilon$-DP KDE mechanism in the function release model for the
Gaussian kernel that, for every dataset $X\subseteq\mathbb R^d$ of size
\begin{equation}\label{eq:dpkde-n}
 n=|X|\ge \Omega\left(\frac{\sqrt{mL}}{\epsilon\alpha}\right),
\end{equation}
returns a function $F_X:\mathbb R^d\to\R$ satisfying
\[
 \forall y\in\mathbb R^d,\qquad
 \Pr\!\left[
   |F_X(y)-\operatorname{KDE}_X(y)|>\alpha
 \right]\le\eta.
\]
The mechanism has initialization time
$O(n(d+m)\log(2d))$, released representation size $O(d+m)$, and evaluation
time $O((d+m)\log(2d))$ per query.
\end{theorem}

In particular, it suffices that
\[
 n\ge \frac{C}{\epsilon\alpha^2}
       \log\!\left(\frac{2d}{\alpha\eta}\right)
       \sqrt{\log\frac{6}{\eta}}
\]
for a sufficiently large absolute constant $C$. The query time is
$\widetilde O(d+\alpha^{-2})$, with only polylogarithmic dependence on
$1/\eta$.

\paragraph{Construction.}
Sample the data-independent feature map $F:\mathbb R^d\to\mathbb R^{2m}$
used in \Cref{thm:kde}, with accuracy $\alpha/2$ and failure probability
$\eta/2$. Observe that the stated choice of $m$ suffices when $C_0$ is large enough. We also choose $C_0$ so that $m\ge L$.
For every realization of $F$ and every $x$,
\begin{equation}
 \|F(x)\|_2=1,
 \qquad
 \|F(x)\|_\infty\le \frac{1}{\sqrt m},
 \qquad
 \|F(x)\|_1\le\sqrt{2m}.
 \label{eq:dp-feature-norms}
\end{equation}

Let $\operatorname{Lap}(b)$ denote the Laplace distribution with scale $b$, whose density is $p(z)=e^{-|z|/b}/(2b)$. Let
\[
 b=\frac{2\sqrt{2m}}{\epsilon} \quad \text{and} \quad  b_0=\frac{2}{\epsilon}.
\]
Independently of $F$, draw a vector $Z\in\R^{2m}$ with iid coordinates from $\operatorname{Lap}(b)$, and an additional scalar $Z_0\sim\operatorname{Lap}(b_0)$. 
Compute
\[
 F(X)=\frac1n\sum_{x\in X}F(x),
 \qquad
 \widetilde n=\max\{1,n+Z_0\},
 \qquad
 \widetilde F(X) = \frac{nF(X)+Z}{\widetilde n} .
\]
Release the description of $F$ and the vector $\widetilde F(X)$. The released function is defined as
\[
\widetilde{F}_X(y) = \widetilde{F}(X)^TF(y). 
\]

\paragraph{Privacy.}
Fix any realization of $F$. Adding or dropping a point $x$ changes $S_X$
by $\pm F(x)$, so \eqref{eq:dp-feature-norms} bounds its $\ell_1$-sensitivity
by $\sqrt{2m}$. The count $n$ has sensitivity~$1$.
Thus the Laplace mechanism makes $\widetilde S_X$ and $n+Z_0$ each
$\epsilon/2$-DP, and their joint release is $\epsilon$-DP by standard DP composition
\cite{dwork2014algorithmic}. This holds for every $F$, and the randomness that determines $F$ itself is oblivious to the dataset, thus the entire release $(F,\widetilde{F}(X))$ is $\epsilon$-DP.

\paragraph{Accuracy.} Fix a query $y\in\R^d$. 
The accuracy analysis involves two sources of error -- the non-private KDE approximation error from \Cref{thm:kde}, and the error due to the privacy-preserving Laplace noise. 
For the non-private KDE error, by \Cref{thm:kde} and the choice of $m$, we have
\begin{equation}
 \Pr\!\left[
   |F(X)^TF(y)-\operatorname{KDE}_X(y)|>\alpha/2
 \right]\le\eta/2.
 \label{eq:dp-nonprivate-error}
\end{equation}

For the privacy noise error, 
recall that the MGF of $A\sim\operatorname{Lap}(b)$ is $\E[e^{tA}]=1/(1-b^2t^2)$ for all $|t|<1/b$. 
By \Cref{eq:dp-feature-norms} we have $|F(y)_j|\leq1/\sqrt{m}$ for every $j=1,\ldots,2m$. 
Therefore, for every $s\in\R$ such that $|s|\leq\sqrt{m}/(2b)$, 
conditioned on the random choice of $F$, 
\[
\log\E[\exp(sZ^TF(y))\mid F] = -\sum_{j=1}^{2m}\log(1-b^2s^2F(y)_j^2) \leq \sum_{j=1}^{2m}2b^2s^2F(y)_j^2 = 2b^2s^2, 
\]
having used that $-\log(1-u)\leq2u$ for all $u\in[0,1/4]$ and that $\norm{F(y)}_2^2=1$ by \Cref{eq:dp-feature-norms}. 
Therefore, using the Chernoff-Hoeffding analysis blueprint, 
\begin{align*}
    \Pr\left[Z^TF(y) > t \mid F\right] &= \Pr\left[\exp(sZ^TF(y))| > \exp(st) \mid F\right] \\
    &\leq \exp(-st)\cdot\E[\exp(sZ^TF(y))\mid F] \\
    &\leq \exp(-st + 2b^2s^2) .
\end{align*}
Observing that $L\leq m$, we may choose $s=\sqrt{L}/(2b)$
and $t=3b\sqrt{L}$ and obtain
\[
  \Pr\left[Z^TF(y) > 3b\sqrt{L} \mid F\right] \leq \exp(-L) = \eta/6 .
\]
The lower tail is similar, hence
\[
  \Pr\left[|Z^TF(y)| > 3b\sqrt{L} \mid F\right] \leq 2\exp(-L) = \eta/3  .
\]
Furthermore, by the standard scalar Laplace distribution tail \cite[Fact 3.7]{dwork2014algorithmic}, 
\[
 \Pr[|Z_0|>b_0L]=e^{-L}=\eta/6 .
\]
Consequently, with probability at least $1-\eta/2$, both of the following events occur:
\begin{equation}\label{eq:dpkde-good-events}
|Z^TF(y)|\le3b\sqrt L \quad \text{and} \quad |Z_0|\le b_0L .
\end{equation}
We assume henceforth that both events hold. 

By choosing a sufficiently large hidden constant in the lower bound on $n$ in \Cref{eq:dpkde-n}, we have $b_0L = 2\log(6/\eta)/\epsilon \leq n/2$. Recalling that $\widetilde{n}=\max\{1,n+Z_0\}$ and $n\geq1$, we have
\begin{equation}\label{eq:dpkde-nbounds}
 \widetilde n\ge n/2
 \quad \text{and} \quad
 |\widetilde n-n|\le |Z_0| \le b_0L.
\end{equation}
Therefore,
\begin{align*}
    \left| \widetilde F_X(y) - F(X)^TF(y)\right| &= \left| \widetilde F(X)^TF(y) - F(X)^TF(y)\right| & \\
    &= \left| \frac{Z^TF(y) + (n-\widetilde n) F(X)^TF(y)}{\widetilde n} \right| & \text{plugging $\widetilde F(X)$ and rearranging} \\
    &\leq \frac{2}{n}\left(3b\sqrt{L} + b_0L\cdot F(X)^TF(y)\right) & \text{by \Cref{eq:dpkde-good-events,eq:dpkde-nbounds}} .
\end{align*}
Observe that
\[
  F(X)^TF(y) = \frac1n\sum_{x\in X}F(x)^TF(y) \leq \frac1n\sum_{x\in X}\norm{F(x)}_2\norm{F(y)}_2 \leq 1,
\]
where we have used Cauchy-Schwartz and then the property $\forall_{x\in\R^d}\norm{F(x)}_2=1$ from \Cref{eq:dp-feature-norms}. Therefore, plugging $b=2\sqrt{2m}/\epsilon$ and $b_0=2/\epsilon$, 
\begin{align*}
  \left| \widetilde F_X(y) - F(X)^TF(y)\right|
  &\leq
  \frac{2}{n}\left(3b\sqrt{L} + b_0L\right) \\
 &=\frac{12\sqrt{2mL}+4L}{n\epsilon}\\
 &\le\frac{(12\sqrt2+4)\sqrt{mL}}{n\epsilon} \\
 &\le\frac{\alpha}{2}.
\end{align*}

This it the error arising from privacy noise. 
Combining this with the non-private KDE error from \eqref{eq:dp-nonprivate-error}, the triangle
inequality and a union bound yield
\[
 \Pr\!\left[
   |\widetilde F_X(y)-\operatorname{KDE}_X(y)|>\alpha
 \right]\le\eta,
\]
as required.

\paragraph{Computational bounds.}
Computing $F(X)$ has the same cost as computing the feature mean in
Theorem~7.1, namely $O(n(d+m)\log(2d))$. Sampling and adding the noise and
normalizing the released vector take $O(m)$ additional time.
The map description and the released vector together occupy $O(d+m)$
space. Evaluating $F(y)$ takes $O((d+m)\log(2d))$ time, and the inner
product $\widetilde F(X)^TF(y)$ takes $O(m)$ additional time.

This completes the proof of \Cref{thm:dp-kde}. \qed

\section*{AI use statement}

All the proofs in this paper were found through interactions with generative AI, primarily GPT 5.6 Sol, and rely heavily on its contribution along with substantial human input. The authors wrote the manuscript manually based on those interactions to ensure correctness, readability and faithfulness to existing literature. They are accountable for the content and correctness.

The authors welcome all feedback to improve the paper.

\section*{Acknowledgments}
TW's research was supported by the Israeli Ministry of Innovation, Science \& Technology, by Len
Blavatnik and the Blavatnik Family foundation, and by an Alon Scholarship.
This work started when the authors attended a workshop on the \emph{Intersections of Differential Privacy and Sublinear Algorithms} at Toyota Technological Institute at Chicago (TTIC), and the authors would like to thank the hosts and organizers for their hospitality.

\bibliography{nearlinrff}
\bibliographystyle{amsalpha}

\appendix

\section{Brascamp-Lieb inequalities}\label{app:bl}
In this appendix, for completeness, we give more precise references for the Brascamp-Lieb inequalities in \Cref{lmm:bl_gaussian,lmm:sphbl}.

\paragraph{Gaussian Brascamp-Lieb (\Cref{lmm:bl_gaussian}).}
This lemma is a slight adaptation of a statement from \cite{barthe2011correlation}, or Theorem 1(i) from \cite{chen2015improved}. 
We show how to derive it from the statement in the bottom of page 5 in the arxiv version \cite{barthe2009correlation} of \cite{barthe2011correlation}. If one takes $n_i=1$ for all $i$, meaning each $E_i$ is a one-dimensional subspace spanned by the unit vector $w_i$, then that statement becomes exactly our lemma except that it requires the equality
\[ \sum_{i=1}^k c_iw_iw_i^T = I_d \]
instead of the relaxed condition in \Cref{eq:subframe}. Showing that the relaxed condition indeed suffices is straightforward: suppose we have \Cref{eq:subframe}. Denote the remainder
\[ R = I_d - \sum_{i=1}^k c_iw_iw_i^T , \]
and write its spectral decomposition
\[ R = \sum_{i=k+1}^{k'}c_iw_iw_i^T, \]
where $c_i>0$ and $w_i$ are unit vectors. Now we have the equality condition
\[
    \sum_{i=1}^{k'} c_iw_iw_i^T = I_d .
\]
For every $i=k+1,\ldots,k'$ we choose $f_i\equiv1$, which obviously doesn't change the conclusion of the lemma.

\paragraph{Spherical Brascamp-Lieb (\Cref{lmm:sphbl}).}
This lemma is exactly the rank-1 case of the \emph{``Geometric Brascamp-Lieb inequality on the sphere''} in \cite{barthe2009correlation}, page 6; that is, where each $E_i$ (in their notation) is the 1-dimensional subspace spanned by $w_i$ (in our notation). We also refer to \cite{barthe2006entropy, carlen2004sharp} for earlier works that prove similar results about spherical Brascamp-Lieb.

\end{document}